\documentclass{article}

\usepackage{microtype}
\usepackage{graphicx}
\usepackage{subfigure}
\usepackage{natbib}
\usepackage{algorithm}
\usepackage[noend]{algorithmic}

\usepackage{hyperref}

\usepackage{booktabs}
\usepackage{multirow, makecell}

\usepackage[accepted]{icml2019_arxiv}

\icmltitlerunning{Learning Optimal Linear Regularizers}

\usepackage{amsmath,amssymb}
\usepackage{amsthm}

\newtheorem{corollary}{Corollary}
\newtheorem{definition}{Definition}
\newtheorem{lemma}{Lemma}
\newtheorem{proposition}{Proposition}
\newtheorem{theorem}{Theorem}

\DeclareMathOperator*{\argmin}{argmin}

\newcommand{\E}{\mathbb{E}}

\newcommand{\ignore}[1]{}

\newcommand{\reals}[0]{\mathbb{R}}

\newcommand{\set}[1]{\ensuremath{\left\{#1\right\}}}
\newcommand{\surl}[1]{\begin{small}\url{#1}\end{small}}
\newcommand{\tup}[1]{\langle#1\rangle}

\newenvironment{varalgorithm}[1]
  {\algorithm\renewcommand{\thealgorithm}{#1}}
  {\endalgorithm}

\newcommand{\concat}{\ensuremath{\oplus}}

\newcommand{\D}{\mathcal{D}}
\newcommand{\DD}{\mathcal{D}_1}
\newcommand{\ee}{\ensuremath{\mbox{ .}}}
\newcommand{\ex}{z}
\newcommand{\exset}{Z}
\newcommand{\feasible}{\mathcal{F}}
\newcommand{\mf}{q}

\newcommand{\R}{\mathcal{R}}
\newcommand{\sas}{\mathsf{SAS}}
\newcommand{\w}{\theta}
\newcommand{\W}{\Theta}
\newcommand{\V}{V}

\newcommand{\fv}{\phi}
\newcommand{\fvz}{\phi_0}
\newcommand{\eL}{\bar L}

\begin{document}

\twocolumn[
\icmltitle{Learning Optimal Linear Regularizers}

\begin{icmlauthorlist}
	\icmlauthor{Matthew Streeter}{google}
\end{icmlauthorlist}

\icmlaffiliation{google}{Google Research}

\icmlcorrespondingauthor{Matthew Streeter}{mstreeter@google.com}

\vskip 0.3in
]

\printAffiliationsAndNotice{}

\begin{abstract}
We present algorithms for efficiently learning regularizers that
improve generalization.
Our approach is based on the insight that regularizers can be viewed
as upper bounds on the generalization gap, and that reducing
the slack in the bound can improve performance on test data.
For a broad class of regularizers,
the hyperparameters that give the best upper bound can be computed
using linear programming.
Under certain Bayesian assumptions, solving the LP lets us ``jump" to
the optimal hyperparameters given very limited data.
This suggests a natural algorithm for tuning regularization hyperparameters,
which we show to be effective on both real and synthetic data.
\end{abstract}

\section{Introduction}

Most machine learning models are obtained by minimizing a loss function,
but optimizing the training loss is rarely the ultimate goal.
Instead, the model is ultimately judged based on information that is
unavailable during training, such as performance on held-out test data.
The ultimate value of a model therefore depends critically on the
loss function one chooses to minimize.

Traditional loss functions used in statistical learning are the sum of two
terms: the empirical training loss and a regularization penalty.
A common regularization penalty is the $\ell^1$ or $\ell^2$
norm of the model parameters.
More recently, it has become common to regularize implicitly by
perturbing the examples, as in dropout \cite{srivastava2014dropout},
or perturbing the
labels, as in label smoothing \cite{szegedy2016rethinking},
or by modifying the training
algorithm, as in early stopping \cite{caruana2001overfitting}.

The best choice of regularizer is usually not obvious \emph{a priori}.
Typically one chooses a regularizer that has worked well on similar
problems, and then fine-tunes it by searching for the hyperparameter values
that give the best performance on a held-out validation set.
Though this approach can be effective, it tends to require a large
number of training runs, and to mitigate this
the number of hyperparameters must be kept fairly small in practice.

In this work, we seek to recast the problem of choosing regularization
hyperparameters as a supervised learning problem which can be solved
more efficiently than is possible with a purely black-box
approach.  Specifically, we show that the optimal regularizer is
by definition the one that provides the tightest possible bound on
the generalization gap (i.e., difference between test and training loss),
for a suitable notion of ``tightness".  We then present an algorithm
that can find approximately optimal regularization hyperparameters
efficiently via linear programming.
Our method applies to explicit regularizers such as L2,
but also in an approximate way to implicit regularizers such as dropout.

Our algorithm takes as input a small set of models for which we have computed
both training and validation loss, and produces a set of recommended
regularization hyperparameters as output.  Under certain Bayesian
assumptions, we show that our algorithm returns the optimal regularization
hyperparameters, requiring data from as few as \emph{two}
training runs as input.
Building on this linear programming algorithm, we present a hyperparameter
tuning algorithm and show that it outperforms state-of-the-art
alternatives on real problems where these assumptions do not hold.

\subsection{Definitions and Notation} \label {sec:definitions}

We consider a general learning problem with an arbitrary loss function.
Our goal is to choose a hypothesis $\w \in \W$,
so as to minimize the expected value of a loss function $\ell(\ex, \w)$
for an example $\ex$ drawn from an unknown distribution $\D$.  That is,
we wish to minimize
\[
	L(\w) = \E_{\ex \sim \D}[\ell(\ex, \w) ] \ee
\]
In a typical supervised learning problem, $\w$ is a parameter vector,
$\ex$ is a (feature vector, label) pair, and $\ell$ is a loss function
such as log loss or squared error.  In an unsupervised problem, $\ex$
might be an unlabeled image or a text fragment.

We assume as input
a set $\set{\ex_i\ |\ 1 \le i \le n}$ of training examples.
Where noted, we assume each $\ex_i$ is sampled independently from $\D$.
We denote average training loss by:
\[
	\hat L(\w) = \frac 1 n \sum_{i=1}^n \ell(\ex_i, \w) \ee
\]

We will focus on algorithms that minimize an objective function
$f(\w) = \hat L(\w) + R(\w)$,
where $R: \W \rightarrow \reals_{> 0}$ is the regularizer (the choice
of which may depend on the training examples).
We denote the minimizer of regularized loss by
\[
	\hat \w = \argmin_{\w \in \W} \set { \hat L(\w) + R(\w) }
\]
and the optimal hypothesis by
\[
	\w^* = \argmin_{\w \in \W} \set {L(\w)} \ee
\]
We refer to the gap $L(\hat \w) - L(\w^*)$ as \emph{excess test loss}.
Our goal is to choose $R$ so that the excess test loss is
as small as possible.

\section{Regularizers and Generalization Bounds} \label {sec:connection}

Our strategy for learning regularizers will be to compute a regularizer that
provides the tightest possible estimate of the generalization gap (difference
between test and training loss).
To explain the approach,
we begin with a mathematically trivial yet surprisingly useful observation:
\[
	R(\w) = L(\w) - \hat L(\w) \implies L(\hat \w) = L(\w^*) \ee
\]
That is, the generalization gap is by definition an optimal regularizer, since
training with this regularizer amounts to training directly on the test loss.
More generally,
for any monotone function $h$, the regularizer $R(\w) = h(L(\w)) - \hat L(\w)$ is optimal.

Though training directly on the test loss is clearly not a workable strategy,
we might hope to use a regularizer that accurately estimates the
generalization gap, so that regularized training loss estimates test loss.  What makes a good estimate in this context?

In supervised learning
we typically seek an estimate with good average-case performance, for example
low mean squared error.  However, that fact that $\hat \w$
is obtained by optimizing over all $\w \in \W$ means that a single bad
estimate could make $L(\hat \w)$ arbitrarily bad, suggesting that worst-case error matters.  At the same time, it should be less important to estimate $L(\w)$
accurately if $\w$ is far from optimal.  The correct characterization turns
out to be:

\begin{quote}
	\emph{A good regularizer is one that provides an upper bound on the generalization
	gap that is tight at near-optimal points.}
\end{quote}

\begin{figure} [h]
	\begin{center}
	\includegraphics[width=3in]{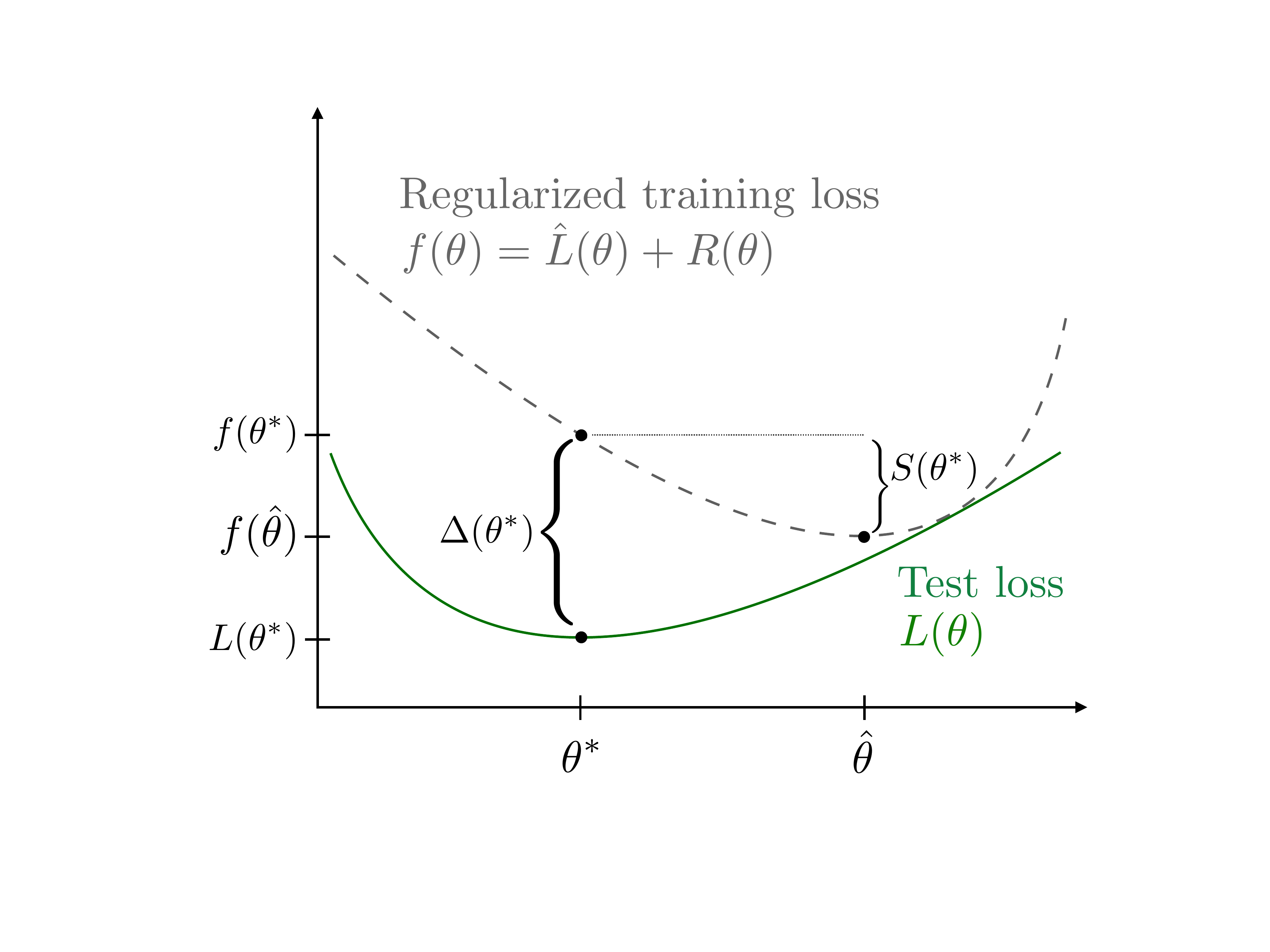}
	\caption{Suboptimality ($S$) and slack ($\Delta$).}
\label{fig:suboptimality}
\end{center}
\end{figure}

To formalize this, we introduce two quantities.
For a fixed regularizer $R$,
defining a regularized training loss $f(\w) = \hat L(\w) + R(\w)$,
define the \emph {slack} of $R$ at a point $\w$ as
\[
	\Delta(\w) \equiv f(\w) - L(\w) \ee
\]
For any $\w$, define the \emph {suboptimality} as
$S(\w) \equiv f(\w) - f(\hat \w)$.
Figure~\ref{fig:suboptimality} illustrates these definitions.

We now give an expression for excess test loss in terms of slack and
suboptimality.  For any $\w$, $L(\w) = f(\w) - \Delta(\w)$, so
\begin{align*}
L(\hat \w) - L(\w) & = f(\hat \w) - \Delta(\hat \w) - f(\w) + \Delta( \w) \\
	& = \Delta(\w) - \Delta(\hat \w) - S(\w) \\
	& \equiv \sas(\w) \mbox { .}
\end{align*}
We refer to the quantity $\sas(\w)$ as \emph{suboptimality-adjusted slack}.
Maximizing both sides over all $\w \in \W$ gives an expression for
excess test loss:
\[
	L(\hat \w) - L(\w^*) = \max_{\w \in \W} \set { \sas(\w) } \ee
\]
That is, the excess test loss of a hypothesis $\hat \w$ obtained by minimizing
$\hat L(\w) + R(\w)$ is the worst-case suboptimality-adjusted slack.
An optimal regularizer
is therefore one that minimizes this quantity.
This is summarized in the following proposition.

\begin{proposition} \label {prop:bound}
For any set $\W$ of hypotheses, and any set $\R$ of regularizers, the
optimal regularizer is
\[
	\argmin_{R \in \R} \set { L(\hat \w(R)) }
	= \argmin_{R \in \R} \set { \max_{\w \in \W} \set { \sas(\w; R) } }
\]
where $\hat \w$ and $\sas$ are defined as above, with the dependence on $R$
now made explicit.
\end{proposition}

How can we make use of Proposition~\ref{prop:bound} in practice?
We do not of course know the test loss for all $\w \in \W$, and thus
we cannot compute $\max_{\w \in \W} \set {\sas(\w; R)}$ exactly.  However,
it is feasible to compute the validation loss for a small set
$\W_0 \subseteq \W$ of hypotheses, for example by doing multiple training runs
with different regularization hyperparameters, or doing a single run with
different thresholds for early stopping.
We can then compute an approximately optimal regularizer using the
approximation:
\begin{equation} \label{eq:approx}
	R^* \approx \argmin_{R \in \R} \set {\max_{\w \in \W_0} \set {\hat \sas(\w; R)} }
\end{equation}
where $\hat \sas$ is an estimate of $\sas$ that uses validation loss as a
proxy for test loss, and uses $\W_0$ as a proxy for $\W$.

Importantly, the $R^*$ given by equation~\ref{eq:approx} will generally
\emph{not} be one of the regularizers we already tried when producing
the models in $\W_0$, as is shown formally in Theorem~\ref{thm:efficiency}.

This suggests a simple iterative procedure for tuning regularization
hyperparameters.  Initially, let $\W_0$ be a small set of hypotheses
trained using a few different hyperparameter settings.  Then, use
approximation \eqref{eq:approx} to compute an approximately optimal
regularizer $R_1$ based on $\W_0$.  Then, train a hypothesis
$\w_1$ using $R_1$, and add it to $\W_0$ to obtain a new set
$\W_1 = \W_0 \cup \set{ \w_1  }$, use $\W_1$ to obtain a better
approximation $R_2$, and so on.

\subsection{Implicit Regularizers}

So far we have assumed the regularizer is an explicit function of the
model parameters, but many regularizers used in deep learning
do not take this form.  Instead, they
operate implicitly by perturbing the weights, labels, or
examples.  Is Proposition~\ref{prop:bound} still useful in this case?

It turns out to be straightforward to accomodate such regularizers.
To illustrate, let $P(\w)$ be a (possibly randomized) perturbation applied
to $\w$.  Training with the loss function $\E[\hat L(P(\w))]$ is equivalent to
using the regularizer
\[
	R(\w) = \E[\hat L(P(\w))] - \hat L(\w) \ee
\]
In the case of dropout, $P(\w)$ sets each activation in
a layer to 0 independently with some probability $p$, which is equivalent
to setting the corresponding outgoing weights to zero.  The regularizer
is simply the expected difference between training loss using the perturbed
weights and training loss using the original weights.  This observation,
together with Proposition~\ref{prop:bound}, yields the following
conclusion:

\begin{quote}
\emph{The best dropout probability is the one that makes the gap between
training loss with and without dropout be the tightest possible estimate of the
generalization gap (where tightness is
worst-case suboptimality-adjusted slack).}
\end{quote}

Similar constructions can be used to accomodate implicit regularizers that
perturb the labels (as in label smoothing) or the input data (as in data
augmentation).
More generally, we can view any perturbation
as a potentially useful regularizer.  For example, training with
quantized model weights can be viewed as using a regularizer equal to
the increase in training loss that quantization introduces.  Quantization
will be effective as a regularizer to the extent that this increase
provides a tight estimate of the generalization gap.

\section {Learning Linear Regularizers} \label {sec:learning}

We now consider the problem of computing an approximately optimal regularizer
from some set $\R$ of possible regularizers,
given as input a set $\W_0$ of hypotheses for which we have computed both
training and validation loss.  In practice, the models in $\W_0$ might be
the result of training for different amounts of time, or using different
hyperparameters.

We will present an algortihm for computing the best \emph{linear regularizer},
defined as follows.
\begin{definition}
A \emph{linear regularizer} is a function of the form
\[
	R(\w; \lambda) = \lambda \cdot \mf(\w)
\]
where $\mf: \W \rightarrow \reals^k$ is a function that, given a model, returns
a feature vector of length $k$.
\end{definition}
Commonly-used regularizers such as L1 and L2 can be expressed as linear
regularizers by including the L1 or L2 norm of the model in the feature vector,
and novel regularizers can be easily defined by including additional features.
We consider the case where $\R = \set{R(\cdot; \lambda): \lambda \in \reals^k}$
is the set of all linear regularizers
using a fixed feature vector $\mf(\w)$.

Dropout is not a linear regularizer, because the implicit regularization
penalty, $R(\w, p) = \E[\hat L(\mathrm{dropout}(\w, p))] - \hat L(\w)$,
varies nonlinearly as a function of the dropout probability $p$.  However, dropout
can be approximated by a linear regularizer using a feature vector of the
form
$\mf(\w) = \langle R(\w, p_1), R(\w, p_2), \ldots, R(\w, p_k) \rangle$,
for a suitably fine grid of dropout probabilties $\set{p_1, p_2, \ldots, p_k}$.

Our algorithm for computing the best linear regularizer is designed to
have two desirable properties:
\begin{enumerate}
  \item \emph{Consistency:} in the limiting case where $\W_0 = \W$,
	  and validation loss is an exact estimate of test loss,
	  it recovers an optimal regularizer.
  \item \emph{Efficiency:} in the case where there exists a regularizer that
    perfectly estimates the generalization gap, we require only $k+1$
    data points in order to recover it.
\end{enumerate}

To describe our algorithm, let $\V(\w)$ be average
loss on a validation set.
To guarantee consistency, it is sufficient that we return
the regularizer $R \in \R$ that minimizes the validation loss of $\hat \w_0(R)$,
where
$\hat \w_0(R)$ is the model in $\W_0$ that minimizes regularized training loss
when $R$ is the regularizer.  That is, we wish to find the
regularizer
\[
   \hat R \equiv \argmin_{R \in \R} \set { \V(\hat \w_0(R))  }
\]
where
$
    \hat \w_0(R) \equiv \argmin_{\w \in \W_0} \set { \hat L(\w) + R(\w) }
$.

This regularizer can be computed as follows.
For each $\w_i \in \W_0$, we solve a linear program to compute a function
of the form $f(\w) = \hat L(\w) + \lambda \cdot \mf(\w)$, subject to
the constraint that $\w_i = \argmin_{\w \in \W_0} \set {f(\w)}$,
or equivalently $f(\w_i) \le f(\w) \ \forall \w \in \W_0$.
Among the $\w_i$'s for which the LP is feasible, the one with
minimum $\V(\w_i)$ determines $\hat R$.  Knowing this,
we consider the $\w_i$'s in ascending order of $V(\w_i)$, stopping as soon
as we find an LP that is feasible.

To guarantee efficiency, we must include additional constraints in our
linear program that break ties in the case where there are multiple values
of $\lambda$ that produce the same argmin of $f$.
As we will show in Theorem~\ref{thm:efficiency}, a sufficient constraint is
that $f$ is an upper bound on
validation loss that minimizes total slack.
Additionally, our linear program gives us the freedom to upper bound
$\alpha \V(\w)$ rather than $\V(\w)$ (for some $\alpha > 0$),
which is necessary for the guarantees
proved in \S\ref{sec:bayes_optimal}.
Pseudocode is given below.

\newcommand{\jg}{i^*}

\begin{varalgorithm}{LearnLinReg}
\begin{algorithmic}
  \caption{}
  \label{alg:llr}
	\STATE {\bfseries Input:} Set of (validation loss, training loss, feature vector) tuples $\set{(\V_i, \hat L_i, \mf_i) \ |\ 1 \le i \le m}$.
   \STATE Sort tuples in ascending order of $\V_i$, and reindex so that
	$\V_1 \le \V_2 \le \ldots \le \V_m$.
   \FOR {$\jg$ from $1$ to $m$}
	\STATE Solve the following linear program:

	\begin{equation*}
\begin{array}{lll}
	\text{minimize} & \sum_{i=1}^m \Delta_i & \\
	\text{subject to}
  & \alpha \ge 0 \\
  & \Delta_i \ge 0 & \forall i \\
  & f_i = \alpha \V_i + \Delta_i & \forall i\\
  & f_i = \lambda \cdot \mf_i + \hat L_i & \forall i\\
  & f_{\jg} \le f_i & \forall i \\
\end{array}
\end{equation*}

	\STATE If the LP is feasible, return $(\alpha, \lambda)$.
   \ENDFOR
   \STATE Return error.
\end{algorithmic}
\end{varalgorithm}

By construction, \ref{alg:llr} returns the linear
regularizer that would have given the best possible validation loss,
when minimizing regularized training loss over all $\w \in \W_0$.
This guarantees consistency, as summarized in Proposition~\ref{prop:consistency}.

\begin{proposition} \label{prop:consistency}
Assuming it terminates successfully, \ref{alg:llr}
returns a pair $(\alpha^*, \lambda^*)$ such
that
\[
	R(\cdot; \lambda^*) = \argmin_{R \in \R} \set { \V(\hat \w_0(R)) } \ee
\]
\end{proposition}

We now consider efficiency.
In the case where there exists a $\lambda^*$ that allows for perfect
estimation of validation loss, Theorem~\ref{thm:efficiency} shows that
\ref{alg:llr} can recover $\lambda^*$ provided $\W_0$ contains
as few as $k+1$ hypotheses, where $k$ is the size of the feature vector.

\begin{theorem} \label{thm:efficiency}
Suppose there exists a perfect regularizer, in the sense that for
some vector $\lambda^*$ and scalar $\alpha^* > 0$,
\[
	\alpha^* \V(\w) = \hat L(\w) + \lambda^* \cdot \mf(\w) \quad \quad \forall \w \in \W \ee
\]
Let
$D = \set{ (\V(\w_i), \hat L(\w_i), \mf(\w_i))\ |\ 1 \le i \le k+1 }$
be a set of tuples such
that the vectors $\tup{\V(\w_i), \hat L(\w_i)} \concat \mf(\w_i)$ are linearly independent,
where $k = |\mf(\w)|$.
Then, $LearnLinReg(D)$ will return $(\alpha^*, \lambda^*)$.
\end{theorem}
\begin{proof}
Under these assumptions, the first LP considered by \ref{alg:llr}
has a feasible point $\alpha = \alpha^*$, $\lambda = \lambda^*$, and $\Delta_i = 0 \ \forall i$.  Because each $\Delta_i$ is constrained to be
non-negative, this point must be optimal, and thus any optimal point must
have $\Delta_i = 0 \ \forall i$.
Thus, any optimal point must satisfy $\alpha \V(\w_i) = \lambda \cdot \mf(\w_i) + \hat L(\w_i) \ \forall i$.  This is a system of linear equations with $k+1$ variables, and
by assumption the equations are linearly independent.
Thus the solution is unique, and the algorithm
returns $(\alpha^*, \lambda^*)$.
\end{proof}

If desired, \ref{alg:llr} can be easily modified to only consider $\lambda$ vectors in
a feasible set $\feasible = \prod_{j=1}^k [\lambda^{\mathrm{min}}_j, \lambda^{\mathrm{max}}_j]$.

\subsection{Hyperparameter Tuning}

\newcommand{\lambdaseen}{\Lambda_{\mathrm{seen}}}
\newcommand{\train}{\mathrm{train}}
\newcommand{\samplehparams}{\mathrm{random\_sample}}
\begin{varalgorithm}{TuneReg}
\begin{algorithmic}
  \caption{}
  \label{alg:tune_reg}
	\STATE {\bfseries Input:} training loss $\hat L(\w)$, validation loss $\V(\w)$,
	feature vector $\mf(\w)$, initial set $\Lambda_0$ of hyperparameter vectors,
	training algorithm $\train(\lambda)$,
	hyperparameter vector sampler $\samplehparams()$.
	\STATE Set $\W_0 \leftarrow \set{\train(\lambda)\ |\ \lambda \in \Lambda_0}$, $\lambdaseen \leftarrow \Lambda_0$.
    \STATE Set $D \leftarrow \set{(\V(\w), \hat L(\w), \mf(\w))\ |\ \w \in \Theta_0}$.
    \WHILE {True}
	\STATE Set $\_, \lambda \leftarrow \mathrm{LearnLinReg}(D)$.
	\STATE If $\lambda \in \lambdaseen$, set $\lambda \leftarrow \samplehparams()$.
	\STATE Set $\w \leftarrow \train(\lambda)$, and set $\lambdaseen \leftarrow \Lambda \cup \set{\lambda}$
	\STATE Set $D \leftarrow D \cup \set{(\V(\w), \hat L(\w), \mf(\w))}$.
    \ENDWHILE
\end{algorithmic}
\end{varalgorithm}

We now describe how to use \ref{alg:llr} for
hyperparameter tuning.  Given an initial set of hyperparameter vectors,
we train using each one, and observe the resulting training and validation loss,
as well as the feature vector for the trained model.  We then feed
this data to \ref{alg:llr} to obtain a vector $\lambda$
of regularization hyperparameters.  We then train using these hyperparameters,
add the results to our dataset, and re-run \ref{alg:llr}.
Experiments using this algorithm are presented in \S\ref{sec:experiments}.

\section {Recovering Bayes-Optimal Regularizers} \label {sec:bayes_optimal}

We have shown that a regularizer is optimal if it can be used to
perfectly predict the generalization gap, and have presented an
algorithm that can efficiently recover a \emph{linear} regularizer if a
perfect one exists.  Do perfect linear regularizers ever exist?
Perhaps surprisingly, the answer is ``yes" for a broad class of Bayesian
inference problems.

As discussed in \S\ref{sec:definitions},
we assume examples are drawn from a distribution $\D$.  In the Bayesian
setting, we assume that $\D$ itself is drawn from
a (known) prior distribution $\DD$.  Given a training dataset $\exset \sim \D^n$, where $\D \sim \DD$, we now
care about the conditional expected test loss:
\[
	\eL(\w, \exset) \equiv E_{\D \sim \DD}[L(\w, \D) | \exset]
\]
where $L(\w, \D) = \E_{\ex \sim \D}[\ell(\ex, \w)]$ as in \S\ref{sec:definitions}.
A Bayes-optimal regularizer is one which minimizes this quantity.
\begin{definition}
Given a training set $\exset \sim \D^n$, where $\D \sim \DD$,
a Bayes-optimal regularizer
is a regularizer $R^*$ that satisfies:
\[
	\argmin_{\w \in \W} \set { \hat L(\w) + R^*(\w) }
	= \argmin_{\w \in \W} \set { \eL(\w, \exset) } \ee
\]
$R^*$ is {\bf perfect} if, additionally, it satisfies the stronger condition
\[
  \hat L(\w) + R^*(\w) = h(\eL(\w, \exset)) \ \forall \w \in \W
\]
for some monotone function $h$.
\end{definition}

Theorem~\ref{thm:exponential} shows that a perfect Bayes-optimal regularizer
exists for density estimation problems where log loss is the loss function,
$\w$ parameterizes an exponential family distribution $\D$ from which examples
are drawn, and $\DD$ is the conjugate prior for $\D$.

\begin{theorem} \label{thm:exponential}
Let $P(\ex | \w)$ be an exponential family distribution with natural
parameter $\eta(\w)$ and conjugate prior
$P(\w)$, and suppose the following hold:
\begin{enumerate}
  \item $\ex \sim \D$, where $\D = P(\ex | \w^*)$.
  \item $\ell(\ex, \w) = -\log P(\ex | \w)$.
  \item $\w^* \sim P(\w)$.
\end{enumerate}
Then, for any training set $\exset \sim \D^n$, $R^*$ is a perfect, Bayes-optimal regularizer, where
\[
	R^*(\w) = - \frac 1 n \log P(\eta(\w)) \ee
\]
\end{theorem}
\ignore{
\begin{proof}[Proof (sketch)]
By assumption, $P(\ex | \w)$ is an exponential family distribution,
meaning that for some functions $h$, $g$, $\eta$, and $T$, we have
\[
	P(\ex | \w) = h(\ex) g(\w) \exp(\eta(\w) \cdot T(\ex)) \ee
\]
The conjugate prior for an exponential family has the form
\[
	P(\eta(\w)) \propto g(\w)^{n_0} \exp(\eta(\w) \cdot \tau_0)
\]
where $\tau_0$ and $n_0$ are hyperparameters.  One of the distinguishing
properties of exponential families is that when $\w^*$ is drawn
from a conjugate prior, the posterior expectation of $T(\ex)$
has a linear form \cite{diaconis1979conjugate}:
\[
	\E_{\w^* \sim P(\w)} [ \E_{\ex \sim P(\ex | \w^*)}[ T(\ex) ]\ |\ \exset] =  \frac {\tau_0 + \sum_{i=1}^n T(\ex_i)} {n_0 + n} \ee
\]
Using these equations and linearity of expectation, some algebra shows that for any $\w$,
\[
	\hat L(\w) + R^*(\w) = \frac { n + n_0} {n} \eL(\w, \exset)  + K
\]
where $K$ is independent of $\w$.  $R^*$ is therefore perfect and Bayes-optimal.
\end{proof}
}

The proof is given in Appendix A.

In the special case where $\w$ is the natural parameter (i.e., $\eta(\w) = \w$),
Theorem~\ref{thm:exponential} gives $R^*(\w) = - \frac 1 n \log P(\w)$.
Using Bayes' rule, it can be shown that
minimizing $\hat L(\w) + R^*(\w)$ is equivalent to maximizing 
the posterior probability of $\w$ (i.e., performing MAP inference).

\subsection {Example: Coin Flips} \label{sec:coins}

Suppose we have a collection of coins, where coin $i$ comes up heads
with unknown probability $p_i$.  Given a training dataset consisting of
outcomes from flipping each coin a certain number of times, we would like
to produce a vector $\w$, where $\w_i$ is an estimate of $p_i$, so as
to minimize expected log loss on test data.
This can be viewed as a highly simplified version of the problem of
click-through rate prediction for online ads
(e.g., see \citet{mcmahan2013ad}).

For each coin, assume $p_i \sim \mathrm{Beta}(\alpha, \beta)$ for some unknown constants $\alpha$ and $\beta$.
Using the fact that the Bernoulli distribution is a member of the exponential
family whose conjugate prior is the Beta distribution,
and the fact that overall log loss is the sum of the log loss for
each coin, we can prove the following as a corollary of
Theorem~\ref{thm:exponential}.

\begin{corollary} \label{cor:coins}
A Bayes-optimal regularizer for the coin flip problem is given by
\[
    \mathrm{LogitBeta}(\w) \equiv - \frac 1 n \sum_i \alpha \log(\w_i) + \beta \log(1-\w_i) \ee
\]
\end{corollary}
Observe that the LogitBeta regularizer is linear,
with feature vector $\tup{- \sum_i \log(\w_i), -\sum_i \log(1-\w_i)}$.

Given a large validation dataset with many independent coins, it can
be shown that validation loss approaches expected test loss.  Thus,
in the limit,
\ref{alg:llr} is guaranteed to recover the optimal hyperparameters
(the unknown $\alpha$ and $\beta$) when using this feature vector.

\begin{figure*} [h]
	\begin{center}
        \includegraphics[width=3in]{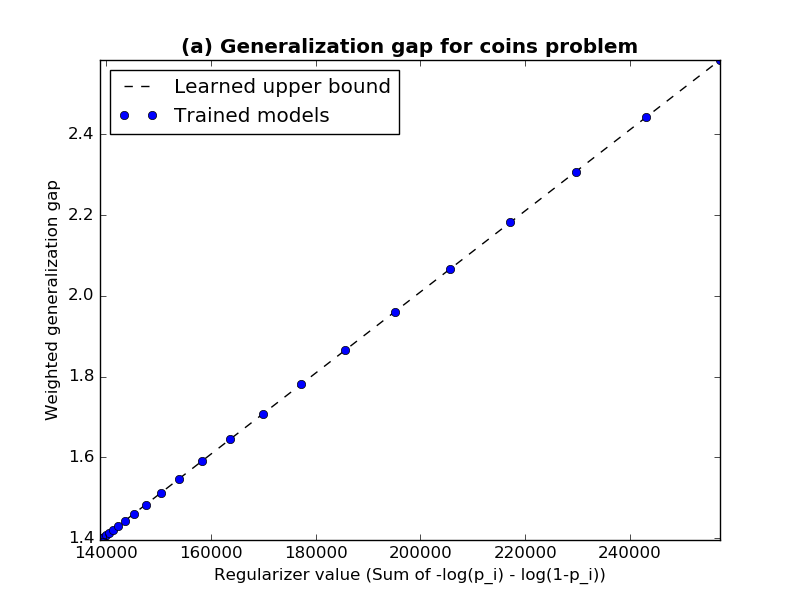}
        \includegraphics[width=3in]{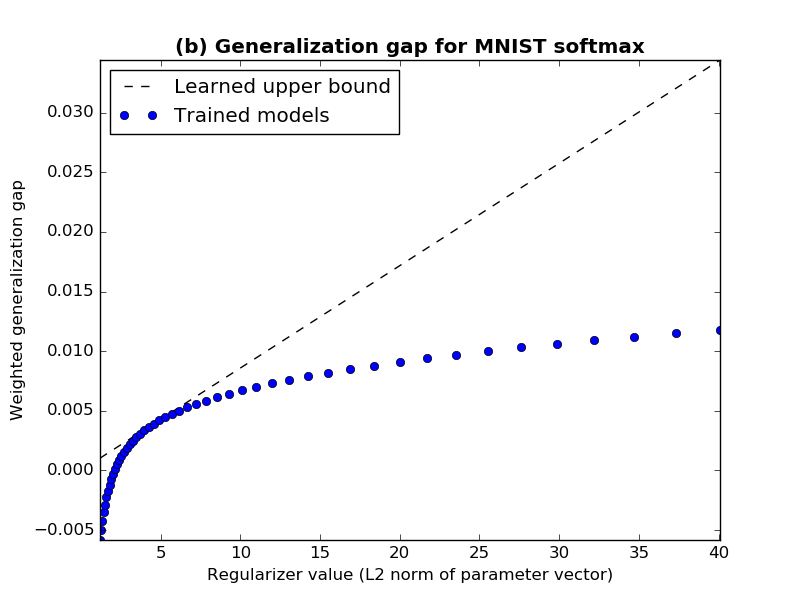}
        \includegraphics[width=3in]{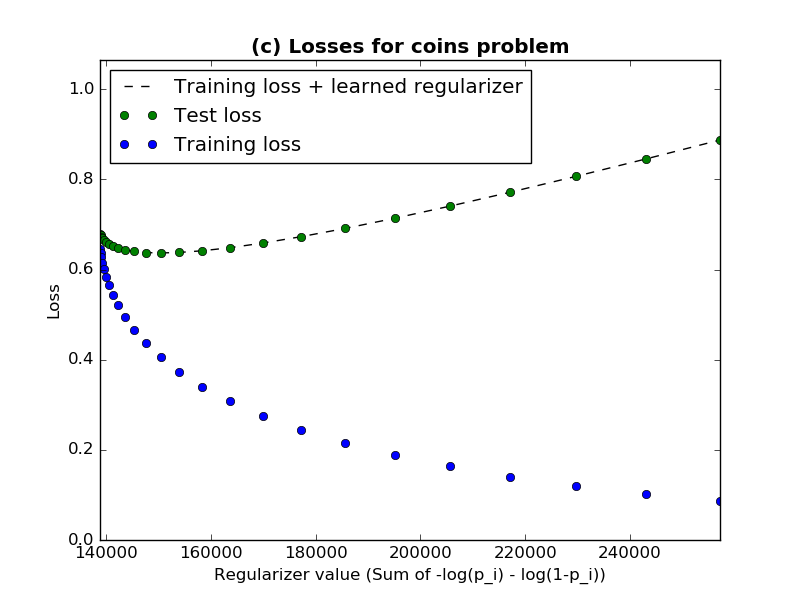}
        \includegraphics[width=3in]{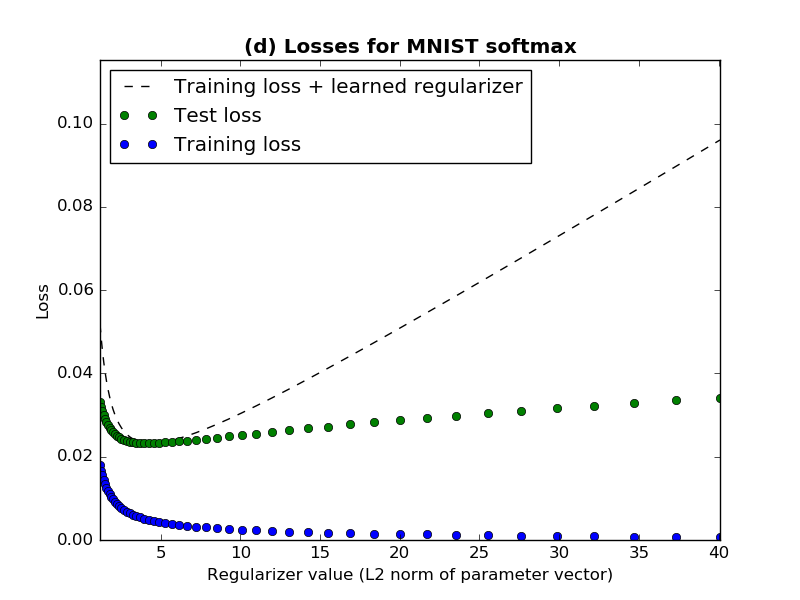}
	\caption{Estimates of the generalization gap for two optimization problems, and corresponding loss functions.  For the coins problem,
		the LogitBeta regularizer perfectly estimates the
		generalization gap for all models.  In contrast,
		for the MNIST softmax
		regression problem, L2 regularization provides an
		upper bound that is only tight for near-optimal models.
		}
\label{fig:learned_regularizers}
\end{center}
\end{figure*}

\begin{table*}[h]
	\caption{Comparison of regularizers, in terms of slack (see Figure~\ref{fig:suboptimality}), test loss, and test accuracy.}
  \label{tab:comparison}
  \centering
	\begin{small}
	\begin{sc}

  \begin{tabular}{llp{1.5cm}p{1.5cm}p{1.5cm}p{2cm}}
    \toprule
			Problem & Regularizer & Max. \newline slack & Max. adj. \newline slack & Min. test \newline loss & Max. test \newline accuracy \\
    \midrule
			Coins($\lambda^*=1$) & LogitBeta & {\bf 0} & {\bf 0} & {\bf 0.637} & -- \\
    \midrule
			\multirowcell{4}[1ex][l]{MNIST \\ Softmax \\ Regression}

  & L1 & 5.83e-1 & 1.65e-3 & 0.0249 & 99.28\% \\
  & L2 & 1.70 & {\bf 1.95e-4} & {\bf 0.0233} & 99.31\% \\
  & Label smoothing & {\bf 1.78e-1} & 2.24e-3 & 0.0254 & 99.31\% \\
  & Dropout (linearized @ $p=.5$) & 6.98e+01 & 2.92e-2 & 0.0463 & {\bf 99.34\%} \\

    \midrule
			\multirowcell{4}[1ex][l]{Inception-v3 \\ Transfer \\ Learning}

  & L1 & {\bf 1.15} & 2.36e-2 & 0.309 & 90.08\% \\
  & L2 & 1.47 & {\bf 1.11e-4} & {\bf 0.285} & {\bf 90.74\%} \\
  & Label smoothing & 5.03 & 8.09e-2 & 0.366 & 89.65\% \\
  & Dropout (linearized @ $p=.5$) & 2.52e+1 & 2.60e-1 & 0.506 & 90.46\% \\

    \bottomrule

  \end{tabular}
  \end{sc}
  \end{small}
\end{table*}

\section {Experiments} \label {sec:experiments}

We now evaluate the \ref{alg:llr} and \ref{alg:tune_reg} algorithms
experimentally, using both real and synthetic data.
Code for both algorithms is included in the supplementary material.

\subsection{Optimization Problems} \label {sec:problems}

We consider three optimization problems.  The first is an instance of
the coin bias estimation problem discussed in
\S\ref{sec:coins}, with $N=10^5$ coins whose true bias is drawn
from a uniform distribution (a Beta distribution with $\alpha = \beta = 1$).
Our training data is the outcome of a \emph{single} flip for each coin,
and we compute the test loss $L(\w)$ exactly.

We then consider two problems that make use of deep networks trained on
MNIST and ImageNet \citep{russakovsky2015imagenet}.
For MNIST, we train a convolutional neural network using a variant of the
LeNet architecture \citep{lecun1998gradient}.
We then consider a softmax regression problem that
involves retraining only the last layer of the network.  This is a convex
optimization problem that we can solve exactly, allowing us to focus on
the impact of the regularizer without worrying about the confounding effects of
early stopping.

We then consider a transfer learning problem.  Starting with an Inception-v3
model trained on ImageNet, we adapt the model to classify images
of flowers from a public dataset \cite{imageretraining2018} split evenly into training and test,
retraining the last layer as in \cite{donahue2014decaf}.

\subsection{Comparison of Regularizers} \label {sec:comparison}

Which regularizers give the tightest generalization bounds?  Does tightness
in terms of slack translate into good test set performance?

To answer these questions, we solve each of the three optimization problems
discussed in \S\ref{sec:problems} using a number of different regularizers.
For the coins problem, we use the LogitBeta regularizer from
Corollary~\ref{cor:coins}, while
for the two softmax regression problems we use L1, L2,
label smoothing, and
a linearized version of dropout.
For label smoothing, the loss function is equivalent to
$\hat L(\w) + \lambda R(\w)$,
where $R(\w)$ is average training loss on a set of uniformly-labeled examples.  For dropout, the regularizer is the
difference between perturbed and unperturbed training loss, with
the dropout probability fixed at .5.

For each problem, we proceed as follows.
For each regularizer
$R_i$, and for each $\lambda$ in a predefined grid $\Lambda_i$,
we generate a model
$\hat \w_{i, \lambda}$ by minimizing $\hat L(\w) + \lambda R_i(\w)$.
For the coins problem, the solution can be found in closed form,
while for the softmax regression problems, we obtain a near-optimal
solution by running AdaGrad for 100 epochs.
Each grid $\Lambda_i$ contains 50 points.  For L1 and L2, the grid is log-uniformly spaced over $[.1, 100]$, while for label smoothing and dropout it is uniform over $[0, 1]$.  The result is a set of models, $\W_i \equiv \set{\hat \w_{i, \lambda}\ |\ \lambda \in \Lambda_i}$.

For each regularizer $R_i$, and each $\w \in \W_i$, we compute the
training loss $\hat L(\w)$, validation loss $\V(\w)$, and regularizer
value $R_i(\w)$.
We then use \ref{alg:llr} to compute an upper bound on validation loss,
using $\mf_i(\w) \equiv \tup{R_i(\w)}$ as the feature vector.
This produces a function of the form
$f_i(\w) = \hat L(\w) + \lambda^*_i R_i(\w)$ such that
$f_i(\w) \ge \alpha_i \V(\w) \ \forall \w \in \W_i$.

Figure~\ref{fig:learned_regularizers} shows the learned upper bounds for
two combinations of problem and regularizer.  In all graphs,
the horizontal axis
is the regularizer value $R_i(\w)$.  The top two graphs compare the
weighted generalization gap, $\alpha_i \V(\w) - \hat L(\w)$, to the learned
upper bound $f_i(\w) - \hat L(\w) = \lambda^*_i R_i(\w)$.
For the coins problem (Figure~\ref{fig:learned_regularizers} (a)),
the learned upper bound is tight for all models, and perfectly
predicts the generalization gap.
In contrast, for the MNIST softmax regression problem with L2 regularization
(Figure~\ref{fig:learned_regularizers} (b)),
the learned upper bound is only tight at near-optimal models.

Figure~\ref{fig:learned_regularizers} (c) and (d) show the corresponding
validation loss and (regularized) training loss.
In both cases, the argmin of
validation loss is very close to the argmin of regularized training loss
when using the learned regularization strength, $\lambda^*_i$.
We see qualitatively similar behavior for the other regularizers on
both softmax regression problems.

Table~\ref{tab:comparison} compares the upper bounds provided by
each regularizers in terms of maximum slack and suboptimality-adjusted
slack.  To make the numbers comparable, the maximum is taken
over \emph{all} trained models (i.e., all $\w \in \bigcup_i \W_i$).  We also
show the minimum test loss and maximum accuracy achieved using
each regularizer.
Observe that:
\begin{itemize}
  \item Except for the coins problem, none of the regularizers produces a upper bound
  whose maximum slack is low (relative to test loss).  However,
L2 regularization achieves uniformly low \emph{suboptimality-adjusted} slack.
  \item The rank ordering of the regularizers in terms of minimum test loss always matches
  the ordering in terms of maximum suboptimality-adjusted slack.  
  \item Dropout achieves high accuracy despite poor slack and log loss,
  suggesting its role is somewhat different than that of the more traditional L1 and L2 regularizers.
\end{itemize}

\subsection{Tuning Regularization Hyperparameters} \label {sec:tune_reg_comparison}

The best values for hyperparameters are typically found
empirically using black-box optimization.  For regularization hyperparameters,
\ref{alg:tune_reg} provides a potentially more efficient way
to tune these hyperparameters, being guaranteed to ``jump" to the optimal
hyperparameter vector in just $k+1$ steps (where $k$ is the number of
hyperparameters) in the special case where a perfect regularizer exists.
Does this theoretical guarantee translate into better performance on real-world
hyperparameter tuning problems?

To answer this question, we compare \ref{alg:tune_reg}
to random search and to Bayesian optimization
using GP-EI-MCMC \cite{snoek2012practical}, on each of the three optimization
problems.  For the coins problem, we use the known optimal LogitBeta
regularizer, while for the two softmax regression problems
we find a linear combination of the regularizers shown in Table~\ref{tab:comparison} (L1, L2, label smoothing,
and linearized dropout).
For each problem, \ref{alg:tune_reg} samples the first $k+1$ points randomly,
where $k$ is the number of hyperparameters.

We consider two variants of each algorithm.  In both cases, random search
and \ref{alg:tune_reg} sample hyperparameter vectors uniformly from a hypercube,
and GP-EI-MCMC uses this hypercube as its feasible set.
In the first variant, the hypercube is based on the full
range of hyperparameter values considered in the previous section.
The second is a more ``informed" variant
based on data collected when generating Table~\ref{tab:comparison}:
the feasible range for label smoothing (where applicable)
is restricted to $[0, .1]$, and log scaling is applied to the L1, L2, and
LogitBeta regularization hyperparameters.  Using log scaling is equivalent
to sampling from a log-uniform distribution for random search and
\ref{alg:tune_reg}, and to tuning the log of the hyperparameter value for
GP-EI-MCMC.

Figure~\ref{fig:tuning} shows the best validation loss achieved by each
algorithm as a function of the number of models trained.
Each curve is an average
of 100 runs.  In all cases, \ref{alg:tune_reg} jumps to near-optimal hyperparameters
on step $k+2$, whether or not the initial $k+1$ random points are sampled
from the informative distribution ($k=1$ in plot (a), and $k=4$ in plots (b) and (c)).  Both variants of \ref{alg:tune_reg} converge much
faster than the competing algorithms, which typically require at least an order of magnitude
more training runs in order to reach the same accuracy.
The results for \ref{alg:tune_reg} can be
very slightly improved by modifying the LP to enforce that all hyperparameters
lie in the feasible range.

\begin{figure}[p]
\begin{center}
	\includegraphics[width=3in]{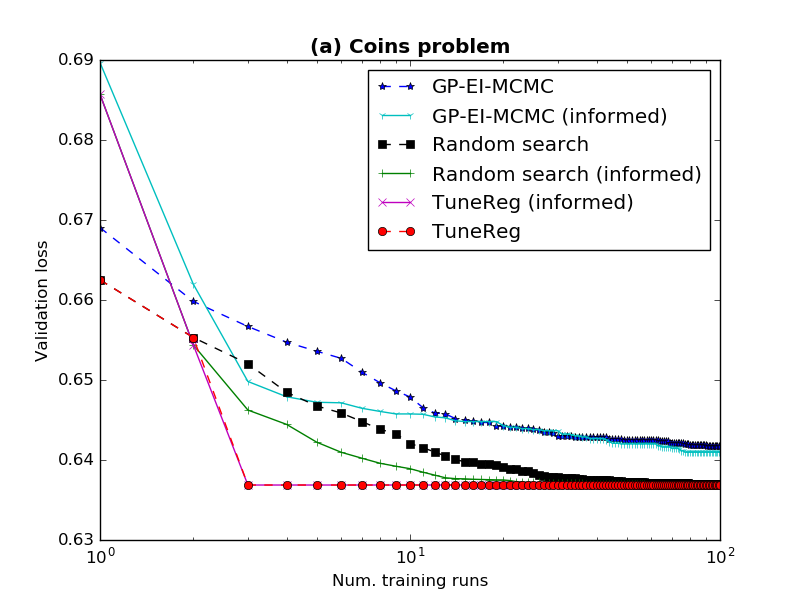}
	\includegraphics[width=3in]{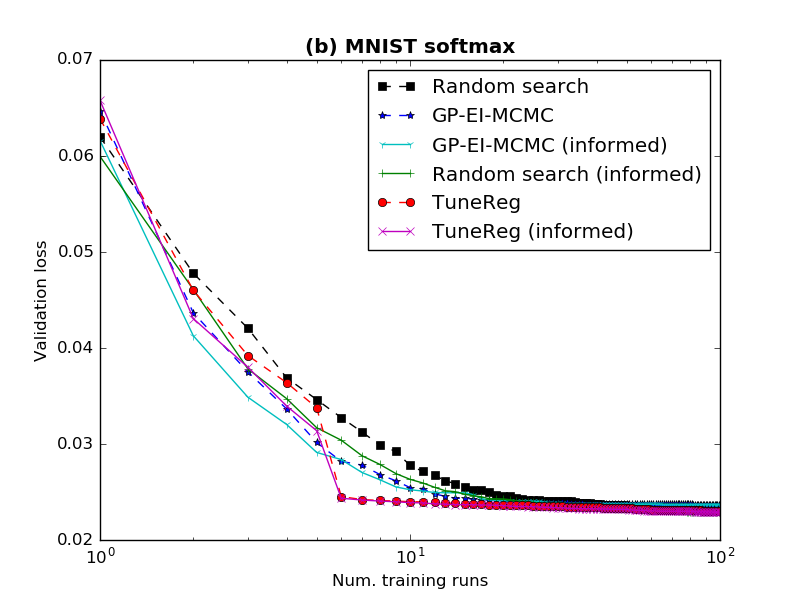}
	\includegraphics[width=3in]{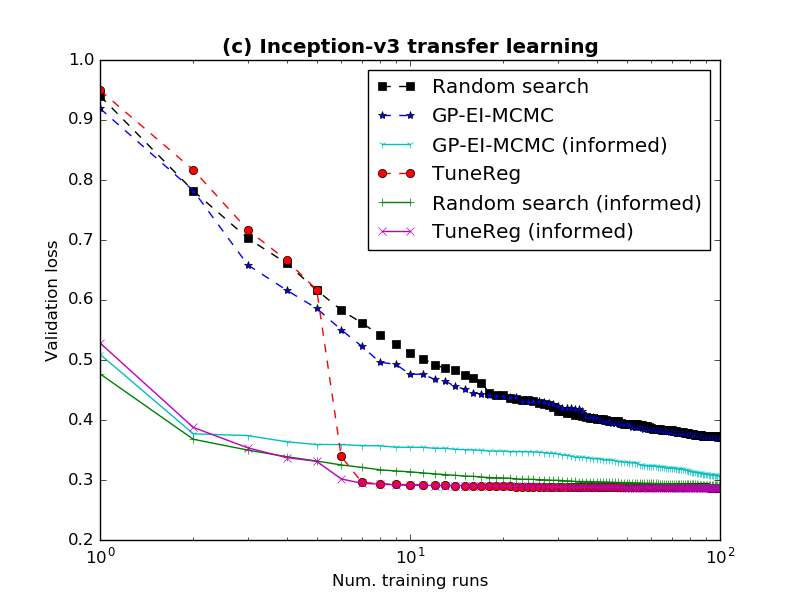}
	\caption{Comparison of algorithms for tuning regularization hyperparameters, with and without informative hyperparameter scaling and feasible range.
	For problems with $k$ hyperparameters, \ref{alg:tune_reg} is able to ``jump" to near-optimal hyperparameters after randomly sampling $k+1$ initial hyperparameter vectors ($k=1$ in plot (a), $k=4$ in plots (b) and (c)).
	}
\label{fig:tuning}
\end{center}
\end{figure}

\section {Related Work}

The high-level idea that a good regularizer should provide an estimate of the
generalization gap appears to be commonly known, though we are not aware of a specific source.  What is novel in our work is the quantitative
characterization of good regularizers in terms of slack and suboptimality,
and the corresponding linear-programming-based algorithm for finding
an approximately optimal regularizer.

Bounding the generalization gap is the subject of a vast literature, which
has focused primarily on worst-case bounds (see \citet{zhang2017understanding}
and references therein).  These bounds have the advantage
of holding with high probability for all models, but are typically too weak
to be used effectively as regularizers.  In contrast, our empirical
upper bounds are only guaranteed to hold for the models we have seen,
but are tight enough to yield improved generalization.
Empirically, \citet{jiang2019predicting} showed that the generalization gap
can be accurately predicted using a linear model based on margin information;
whether this can be used for better regularization is an interesting open question.

An alternative to \ref{alg:tune_reg} is to use gradient-based methods
which (approximately) differentiate validation loss with respect to the
regularization hyperparameters \cite{pedregosa2016hyperparameter}.
Though this idea appears promising, current methods have
not been shown to be effective on problems with more than one hyperparameter,
and have not produced improvements as dramatic as the ones shown in Figure~\ref{fig:tuning}.

\section{Conclusions}

We have shown that the best regularizer is the one that gives the tightest bound on the generalization gap, where
tightness is measured in terms of \emph{suboptimality-adjusted slack}.
We then presented the \ref{alg:llr} algorithm, which computes approximately optimal
hyperparameters for a linear regularizer using linear programming.
Under certain Bayesian assumptions, we showed that \ref{alg:llr} recovers an optimal
regularizer given data from only $k+1$ training runs, where $k$ is the number of hyperparameters.
Building on this, we presented the \ref{alg:tune_reg} algorithm for tuning regularization hyperparameters, and showed that it outperforms state-of-the-art alternatives on both real and
synthetic data.

Our experiments have only scratched the surface of what is possible using our
high level approach.
Promising
areas of future work include (a) attempting to discover novel regularizers, for example by making a
larger number of basis features available to the \ref{alg:llr} algorithm, and (b) adjusting regularization
hyperparameters on-the-fly during training, for example using an online variant of \ref{alg:tune_reg}.

\clearpage
\bibliography{learnreg}

\begin{thebibliography}{13}
\providecommand{\natexlab}[1]{#1}
\providecommand{\url}[1]{\texttt{#1}}
\expandafter\ifx\csname urlstyle\endcsname\relax
  \providecommand{\doi}[1]{doi: #1}\else
  \providecommand{\doi}{doi: \begingroup \urlstyle{rm}\Url}\fi

\bibitem[Caruana et~al.(2001)Caruana, Lawrence, and
  Giles]{caruana2001overfitting}
Caruana, R., Lawrence, S., and Giles, C.~L.
\newblock Overfitting in neural nets: Backpropagation, conjugate gradient, and
  early stopping.
\newblock In \emph{Advances in Neural Information Processing Systems 13}, pp.\
  402--408, 2001.

\bibitem[Diaconis \& Ylvisaker(1979)Diaconis and
  Ylvisaker]{diaconis1979conjugate}
Diaconis, P. and Ylvisaker, D.
\newblock Conjugate priors for exponential families.
\newblock \emph{The Annals of Statistics}, pp.\  269--281, 1979.

\bibitem[Donahue et~al.(2014)Donahue, Jia, Vinyals, Hoffman, Zhang, Tzeng, and
  Darrell]{donahue2014decaf}
Donahue, J., Jia, Y., Vinyals, O., Hoffman, J., Zhang, N., Tzeng, E., and
  Darrell, T.
\newblock De{CAF}: A deep convolutional activation feature for generic visual
  recognition.
\newblock In \emph{Proceedings of the 31st International Conference on Machine
  Learning}, pp.\  647--655, 2014.

\bibitem[Jiang et~al.(2019)Jiang, Krishnan, Mobahi, and
  Bengio]{jiang2019predicting}
Jiang, Y., Krishnan, D., Mobahi, H., and Bengio, S.
\newblock Predicting the generalization gap in deep networks with margin
  distributions.
\newblock In \emph{Proceedings of the International Conference on Learning
  Representations (ICLR)}, 2019.

\bibitem[LeCun et~al.(1998)LeCun, Bottou, Bengio, and
  Haffner]{lecun1998gradient}
LeCun, Y., Bottou, L., Bengio, Y., and Haffner, P.
\newblock Gradient-based learning applied to document recognition.
\newblock \emph{Proceedings of the IEEE}, 86\penalty0 (11):\penalty0
  2278--2324, 1998.

\bibitem[McMahan et~al.(2013)McMahan, Holt, Sculley, Young, Ebner, Grady, Nie,
  Phillips, Davydov, Golovin, et~al.]{mcmahan2013ad}
McMahan, H.~B., Holt, G., Sculley, D., Young, M., Ebner, D., Grady, J., Nie,
  L., Phillips, T., Davydov, E., Golovin, D., et~al.
\newblock Ad click prediction: a view from the trenches.
\newblock In \emph{Proceedings of the 19th ACM SIGKDD International Conference
  on Knowledge Discovery and Data Mining}, pp.\  1222--1230. ACM, 2013.

\bibitem[Pedregosa(2016)]{pedregosa2016hyperparameter}
Pedregosa, F.
\newblock Hyperparameter optimization with approximate gradient.
\newblock In Balcan, M.~F. and Weinberger, K.~Q. (eds.), \emph{Proceedings of
  The 33rd International Conference on Machine Learning}, volume~48 of
  \emph{Proceedings of Machine Learning Research}, pp.\  737--746. PMLR, 2016.

\bibitem[Russakovsky et~al.(2015)Russakovsky, Deng, Su, Krause, Satheesh, Ma,
  Huang, Karpathy, Khosla, Bernstein, Berg, and
  Fei-Fei]{russakovsky2015imagenet}
Russakovsky, O., Deng, J., Su, H., Krause, J., Satheesh, S., Ma, S., Huang, Z.,
  Karpathy, A., Khosla, A., Bernstein, M., Berg, A.~C., and Fei-Fei, L.
\newblock {ImageNet Large Scale Visual Recognition Challenge}.
\newblock \emph{International Journal of Computer Vision (IJCV)}, 115\penalty0
  (3):\penalty0 211--252, 2015.
\newblock \doi{10.1007/s11263-015-0816-y}.

\bibitem[Snoek et~al.(2012)Snoek, Larochelle, and Adams]{snoek2012practical}
Snoek, J., Larochelle, H., and Adams, R.~P.
\newblock Practical bayesian optimization of machine learning algorithms.
\newblock In \emph{Advances in Neural Information Processing Systems 25}, pp.\
  2951--2959, 2012.

\bibitem[Srivastava et~al.(2014)Srivastava, Hinton, Krizhevsky, Sutskever, and
  Salakhutdinov]{srivastava2014dropout}
Srivastava, N., Hinton, G., Krizhevsky, A., Sutskever, I., and Salakhutdinov,
  R.
\newblock Dropout: a simple way to prevent neural networks from overfitting.
\newblock \emph{The Journal of Machine Learning Research}, 15\penalty0
  (1):\penalty0 1929--1958, 2014.

\bibitem[Szegedy et~al.(2016)Szegedy, Vanhoucke, Ioffe, Shlens, and
  Wojna]{szegedy2016rethinking}
Szegedy, C., Vanhoucke, V., Ioffe, S., Shlens, J., and Wojna, Z.
\newblock Rethinking the {I}nception architecture for computer vision.
\newblock In \emph{Proceedings of the IEEE Conference on Computer Vision and
  Pattern Recognition}, pp.\  2818--2826, 2016.

\bibitem[{The Tensorflow Authors}(2018)]{imageretraining2018}
{The Tensorflow Authors}.
\newblock How to retrain and image classifier for new categories.
\newblock \url{https://www.tensorflow.org/hub/tutorials/image_retraining},
  2018.

\bibitem[Zhang et~al.(2017)Zhang, Bengio, Hardt, Recht, and
  Vinyals]{zhang2017understanding}
Zhang, C., Bengio, S., Hardt, M., Recht, B., and Vinyals, O.
\newblock Understanding deep learning requires rethinking generalization.
\newblock In \emph{Proceedings of the International Conference on Learning
  Representations (ICLR)}, 2017.

\end{thebibliography}
\bibliographystyle{icml2019}

\section*{Appendix A: Additional Proofs}

To prove Theorem 2 we will use the following lemma.

\begin{lemma} \label {lem:sufficient}
Suppose that for some functions $\mf$ and $\fv$, the loss function
is of the form:
\[
	\ell(\ex, \w) = \mf(\w) \cdot \fv(\ex) \mbox { .}
\]
Furthermore, suppose there exist constants $n_0$ and $\fvz$ such that,
for any training set $\exset = \set{\ex_i\ |\ 1 \le i \le n}$,
where $\exset \sim \D^n$ and $\D \sim \DD$,
\[
	\E_{\D \sim \DD} [ \E_{\ex \sim \D} [ \fv(\ex) ] \ |\ \exset ] =  \frac {\fvz + \sum_{i=1}^n \fv(\ex_i)} {n + n_0} \ee
\]
Then, there exists a perfect, Bayes-optimal regularizer of the form:
\[
	R^*(\w) = \frac 1 n \mf(\w) \cdot \fvz \mbox { .}
\]
\end{lemma}
\begin{proof}
Let $	\eL(\w, \exset) \equiv \E_{D \sim \DD} [ E_{\ex \sim D}[\ell( \ex, \w )]\ |\ \exset ]$ be the conditional expected test loss.
By linearity of expectation,
\begin{align*}
	\eL(\w, \exset)
	& = \mf(\w) \cdot E_{D \sim \DD} [ E_{\ex \sim D}[\fv(\ex)]\ |\ \exset ]\\
	& = \mf(\w) \cdot \frac {\fvz + \sum_{i=1}^n \fv(\ex_i) } {n + n_0} \ee
\end{align*}
Meanwhile, average training loss is
$
	\hat L(\w) = \frac 1 n \mf(\w) \cdot \sum_{i=1}^n \fv(\ex_i)
$.
Thus,
\[
	(n + n_0) \eL(\w, \exset) - n \hat L(\w) =  n R^*(\w) \ee
\]
Rearranging, $\hat L(\w) + R^*(\w) = \frac {n + n_0} {n} \eL(\w, \exset)$, so $R^*$ is perfect and Bayes-optimal.
\end{proof}
\ignore{
\begin{proof}[Proof (sketch)]
Let $\eL(\w, \exset) = \E_{\D \sim \DD }[L(\w, \D)\ |\ \exset]$ denote the expected test
loss, given the training dataset $\exset$.  Using linearity of expectation,
it is straightforward to show
\[
	(n + n_0) \eL(\w, \exset) - \hat L(\w) = \mf(\w) \cdot \fvz =  R^*(\w) \ee
\]
This in turn implies $R^*$ is Bayes-optimal.
\end{proof}
}

\begin{proof}[Proof of Theorem 2]
By assumption, $P(\ex | \w)$ is an exponential family distribution,
meaning that for some functions $h$, $g$, $\eta$, and $T$, we have
\[
	P(\ex | \w) = h(\ex) g(\w) \exp(\eta(\w) \cdot T(\ex)) \mbox { .}
\]
Setting $\mf(\w) = \tup{-\log g(\w)} \concat -\eta(\w)$
and $\fv(\ex) = \tup{1} \concat T(\ex)$, we have
\[
	-\log(P(\ex | \w)) = \mf(\w) \cdot \fv(\ex) - \log h(\ex) \ee
\]
Because the $-\log h(\ex)$ term does not depend on $\w$, minimizing
$-\log(P(\ex | \w))$ is equivalent to using the loss function
$\ell(\ex, \w) = \mf(\w) \cdot \fv(\ex)$.

The conjugate prior for an exponential family has the form
\[
	P(\eta(\w)) = \frac {1} {Z_0} g(\w)^{n_0} \exp(\eta(\w) \cdot \tau_0)
\]
where $\tau_0$ and $n_0$ are hyperparameters.  One of the distinguishing
properties of exponential families is that when $\w^*$ is drawn
from a conjugate prior, the posterior expectation of $T(\ex)$
has a linear form \cite{diaconis1979conjugate}:
\[
	\E_{\w^* \sim P(\w)} [ \E_{\ex \sim P(\ex | \w^*)}[ T(\ex) ]\ |\ \exset] =  \frac {\tau_0 + \sum_{i=1}^n T(\ex_i)} {n_0 + n} \ee
\]
Thus if we set $\fvz = \tup{n_0} \concat \tau_0$,
\[
	\E_{\D \sim \DD} [ \E_{\ex \sim \D} [ \fv(\ex) ] \ |\ \exset ] =  \frac {\fvz + \sum_{i=1}^n \fv(\ex_i)} {n + n_0} \ee
\]
Lemma~\ref{lem:sufficient} then shows that a perfect regularizer is:
\begin{align*}
	R^*_1(\w)
	& = \frac 1 n \mf(\w) \cdot \fvz \\
	& = \frac 1 n \left ( -n_0 \log (g(\w)) - \tau_0 \cdot \eta(\w) \right ) \\
	& = \frac 1 n \left ( -\log P(\eta(\w)) - \log(Z_0) \right ) \ee
\end{align*}
Because $R^*_1$ and $R^*$ differ by a constant, $R^*$ is also perfect.
\end{proof}

\end{document}